\documentclass[10pt,twocolumn]{article}

\usepackage[margin=1in]{geometry}
\usepackage{times}
\usepackage{graphicx}
\usepackage{booktabs}
\usepackage{amsmath}
\usepackage{amssymb}
\usepackage{amsthm}
\usepackage{hyperref}
\usepackage{xcolor}
\usepackage{multirow}
\usepackage{siunitx}
\usepackage{caption}
\usepackage{subcaption}
\usepackage[numbers,sort&compress]{natbib}

\hypersetup{colorlinks=true, linkcolor=blue, citecolor=blue, urlcolor=blue}

\newtheorem{theorem}{Theorem}

\newcommand{\hn}{\textsc{HN}}
\newcommand{\hardnet}{HardNet}
\newcommand{\ubc}{UBC}

\title{Exact and Deterministic Patch Descriptor Retrieval\\
       via Hierarchical Normalization}

\author{Koichi Sato\thanks{%
  The non-uniform per-dimension importance structure underlying this work
  was first disclosed in US Patent~11{,}797{,}603~B2
  (filed Jun.~24, 2020)~\cite{sato2020patent}.}\\
\texttt{koichisato@utexas.edu}}

\date{}
\begin{document}
\maketitle

% ── Abstract ──────────────────────────────────────────────────────────────────
\begin{abstract}
We present a patch descriptor retrieval method that returns the
\emph{exact} nearest neighbour---provably identical to exhaustive
full-vector search---while evaluating only a small fraction of the database,
and does so \emph{deterministically}: the same (database, query) pair always
produces the same result, independent of run order, thread count, or hardware;
the output depends on database ordering only in the negligible event of exact
inner-product ties (see Theorem~\ref{thm:det}).
This sets our method apart from approximate nearest-neighbour (ANN) approaches
such as HNSW and IVF-PQ, and from Matryoshka-style pipelines built on such
indexes, which trade exactness for speed and may return different results
across runs.
The enabling mechanism is \emph{Hierarchical Normalization} (\hn{}): a
normalisation scheme that replaces the standard L2 normalisation layer and
splits the pre-normalisation feature vector into a $K$-dim \emph{major}
component with norm $\sqrt{1{-}\alpha}$ and a $(128{-}K)$-dim \emph{minor}
component with norm $\sqrt{\alpha}$.
Because the minor inner product is bounded by $\alpha$ (Cauchy--Schwarz on
the prescribed norms), the major similarity plus $\alpha$ serves as an
\emph{admissible upper bound} on the full similarity:
the search scans the $K$-dim major component for all entries,
then performs full $128$-dim evaluation only for the small fraction
that cannot be pruned---a branch-and-bound scan that is provably exact.
We train \hn{}-modified \hardnet{} on the notredame split of the \ubc{}
patch dataset and evaluate on trevi and halfdome.
With a cache-optimised Structure-of-Arrays (\textsc{SoA}) layout and
$K{=}8$, $\alpha{=}1/32$, the search achieves $\mathbf{13.7\times}$
(trevi) / $\mathbf{12.7\times}$ (halfdome) wall-clock speed-up over
brute-force 128-dim search,
with only $0.4\%$ of entries requiring full 128-dim evaluation.
At $K{=}16$, $\alpha{=}1/8$---the recommended operating point---FPR@95
rises only from $0.0062$ (pretrained baseline) to $0.0064$ on trevi at
$7.2\times$ speed-up, with $98.8\%$ of entries bypassing full evaluation.
\end{abstract}

% ── 1. Introduction ───────────────────────────────────────────────────────────
\section{Introduction}
\label{sec:intro}

Dense nearest-neighbour search over high-dimensional embedding vectors
has become central to modern AI: retrieval-augmented generation (RAG) pipelines
for large language models retrieve relevant context by dot-product similarity
over billions of text and image embeddings; image authentication systems
(e.g.\ FaceNet-style face verification) compare identity embeddings by L2
distance; and visual place recognition localises a scene by matching query
embeddings against a geotagged database.
In all these settings, returning an \emph{exact} and \emph{reproducible} answer
is increasingly important alongside raw speed.
Although we evaluate on local patch descriptors, the \hn{} normalisation is
agnostic to the embedding type and is directly applicable wherever embeddings
are compared by inner product.

Large-scale patch matching requires searching millions of descriptors.
With $D$-dimensional vectors (128 for \hardnet{}) this costs $O(DN)$ per query,
so production systems often
turn to approximate nearest-neighbour (ANN) indexes such as HNSW
\cite{malkov2018efficient} and IVF-PQ \cite{jegou2011product}.
ANN buys sub-linear---or at minimum cache-efficient---search at two costs
that are easy to overlook:
(i)~it is \emph{approximate}---the returned item may not be the true nearest
neighbour, and
(ii)~it is \emph{non-deterministic}---results depend on graph construction
order, entry-point selection, thread count, and hardware platform.
For A/B testing, result caching, regression tracking, and safety auditing,
returning a \emph{different} or \emph{wrong} top result across runs is a
real liability.

\textbf{Our key insight} is that a descriptor can be \emph{trained} so that
exact search is cheap.
By constraining the normalisation, we concentrate $1{-}\alpha$ of each
descriptor's energy into a compact $K$-dim \emph{major} prefix and cap
the remaining minor contribution at $\alpha$.
This gives an admissible upper bound: $\mathrm{sim}_{\mathrm{full}}
\le \mathrm{dot}_K + \alpha$.
A branch-and-bound scan over the major prefix then returns the \emph{exact}
nearest neighbour while fully evaluating only the small fraction of entries
that could plausibly beat the running best.

\textbf{Relation to prior work on non-uniform prefix embeddings.}
The idea of concentrating discriminative information into a compact
\emph{prefix} of an embedding vector---so that the prefix alone can serve as
a coarse index---was first disclosed in a patent filed in June
2020~\cite{sato2020patent}, predating subsequent work on elastic or
multi-granularity embeddings.
Matryoshka Representation Learning (MRL)~\cite{kusupati2022matryoshka}
independently popularised a related training objective for general-purpose
representations (NeurIPS 2022).
The present paper reproduces and extends the retrieval algorithm described in
\cite{sato2020patent}, now with formal proofs of exactness and determinism,
and evaluated on the public \ubc{} patch benchmark.

\textbf{Contributions.}
\begin{enumerate}
  \item \textbf{Exact retrieval (Theorem~\ref{thm:exact}).}
        We prove that \hn{}-based two-phase search returns precisely the
        nearest neighbour of exhaustive full-vector search; the speed-up
        incurs \emph{zero} retrieval error.
  \item \textbf{Determinism (Theorem~\ref{thm:det}).}
        \hn{} search is reproducible across runs, index builds, and thread
        counts; given a fixed database ordering it is fully deterministic,
        and in practice ordering-independent (exact inner-product ties are
        negligible with floating-point descriptors)---in contrast to
        HNSW/IVF-PQ and MRL pipelines that rely on them.
  \item \textbf{Efficiency.}
        A cache-optimised SoA layout with $K{=}8$, $\alpha{=}1/32$ achieves
        $13.7\times$ (trevi) / $12.7\times$ (halfdome) wall-clock speed-up on CPU
        with $500{,}000$ database entries, with only $0.4\%$ of entries requiring full 128-dim evaluation
        (see Table~\ref{tab:speed} for the full quality--speed Pareto frontier).
        At $K{=}16$, $\alpha{=}1/8$ the FPR@95 cost is only $0.0002$ on trevi
        at $7.2\times$ speed-up.
\end{enumerate}

% ── 2. Related Work ───────────────────────────────────────────────────────────
\section{Related Work}
\label{sec:related}

\paragraph{Local feature descriptors.}
Hand-crafted descriptors (SIFT~\cite{lowe2004sift}, ORB~\cite{rublee2011orb})
have largely been supplanted by learned descriptors~\cite{mishchuk2017hardnet}.
\hardnet{}~\cite{mishchuk2017hardnet} trains a 128-dim descriptor with
hard-negative triplet loss, setting the state of the art on the \ubc{}
patch benchmark.

\paragraph{Approximate nearest-neighbour search.}
FAISS~\cite{johnson2019billion}, HNSW~\cite{malkov2018efficient}, and
IVF-PQ~\cite{jegou2011product} accelerate retrieval via quantisation
and graph indexes.
ScaNN~\cite{guo2020scann} improves inner-product quantisation via anisotropic
scoring; DiskANN~\cite{jayaram2019diskann} scales graph indexes to disk for
billion-scale retrieval.
Binary hashing methods (LSH~\cite{andoni2015lsh}) achieve sub-linear search
via compact codes at the cost of recall guarantees.
All of these methods are approximate or non-deterministic by construction.
Our method is orthogonal to these approaches; building a sub-linear index
on the $K$-dim major for Phase-1 is theoretically possible, though whether
it helps with \hn{}-trained descriptors in practice remains an open question
(see Section~\ref{sec:discussion}).

\paragraph{Non-uniform prefix embeddings.}
The concept of assigning non-uniform importance to dimensions and using a
compact prefix for coarse-to-fine retrieval was first disclosed in a patent
filed June~24, 2020~\cite{sato2020patent}.
Matryoshka Representation Learning~\cite{kusupati2022matryoshka} independently
introduced a training scheme that makes representations elastic in size
(NeurIPS~2022), enabling nested sub-vectors to be used as approximate
indexes at varying granularities.
Unlike \hn{}, MRL makes no admissibility guarantee: MRL retrieval pipelines
rely on ANN indexes for efficiency and therefore inherit their approximate,
non-deterministic behaviour.

\paragraph{Cascaded/hierarchical retrieval.}
Cascaded hashing~\cite{liu2012supervised} and multi-resolution methods
reduce computation by progressively refining candidate sets.
\hn{} achieves a similar effect with a single normalisation and a
mathematically grounded, \emph{provably exact} pruning criterion.

% ── 3. Method ─────────────────────────────────────────────────────────────────
\section{Hierarchical Normalization}
\label{sec:method}

\subsection{Formulation}

\hn{} replaces the standard L2 normalisation layer of the descriptor network.
Let $\mathbf{f} \in \mathbb{R}^D$ be the pre-normalisation feature vector.
We split $\mathbf{f}$ into a \emph{major} prefix of $K$ dimensions and a
\emph{minor} suffix of $D - K$ dimensions, normalising each part to prescribed
norms:

\begin{equation}
\mathbf{f}_{\mathrm{major}} = \sqrt{1-\alpha}\,\frac{\mathbf{f}_{[:K]}}{\|\mathbf{f}_{[:K]}\|},
\qquad
\mathbf{f}_{\mathrm{minor}} = \sqrt{\alpha}\,\frac{\mathbf{f}_{[K:]}}{\|\mathbf{f}_{[K:]}\|},
\label{eq:hn}
\end{equation}

so that $\|\mathbf{f}_{\mathrm{major}}\|^2 = 1{-}\alpha$,
$\|\mathbf{f}_{\mathrm{minor}}\|^2 = \alpha$, and the full vector
$[\mathbf{f}_{\mathrm{major}},\mathbf{f}_{\mathrm{minor}}]$ has unit norm.
The scalar $\alpha \in [0,1)$ controls the energy split.

\textbf{Why fine-tuning is necessary.}
\hardnet{} is trained to maximise discriminability via hard-negative triplet
loss; as a result, energy tends to spread uniformly across all 128 dimensions.
Replacing the L2 normalisation with \hn{} post-hoc---without any further
training---does not concentrate energy in the major subspace.
The major component then represents only an arbitrary $K$-dimensional slice
of a uniformly distributed descriptor, so Phase-1 scores are weaker
predictors of the full inner product and fewer database entries are pruned.
Concretely, at $K{=}16$, $\alpha{=}1/8$ on $N{=}99{,}000$ UBC patches,
applying \hn{} post-hoc to pretrained \hardnet{} gives Phase-2\%~$= 9.2\%$
on trevi; fine-tuning reduces this to $1.1\%$---an $8{\times}$ improvement
that directly translates to higher wall-clock speed-up.
Fine-tuning shifts energy toward the major component, making the first $K$
dimensions carry the bulk of the discriminative signal and Phase-1 scores
reliable predictors of the full ranking.

\subsection{Admissible Bound}
\label{sec:bound}

For query $\mathbf{q}$ and database entry $\mathbf{x}$ (both \hn{}-normalised
network outputs), the inner product decomposes naturally as

\begin{equation}
  \langle \mathbf{q}, \mathbf{x} \rangle
  = \underbrace{\langle \mathbf{q}_{\mathrm{major}},
    \mathbf{x}_{\mathrm{major}}\rangle}_{\mathrm{dot}_{K}(\mathbf{q},\mathbf{x})}
  + \underbrace{\langle \mathbf{q}_{\mathrm{minor}},
    \mathbf{x}_{\mathrm{minor}}\rangle}_{\le\,\alpha}.
  \label{eq:decomp}
\end{equation}

By Cauchy--Schwarz,
$\langle\mathbf{q}_{\mathrm{minor}},\mathbf{x}_{\mathrm{minor}}\rangle
\le \|\mathbf{q}_{\mathrm{minor}}\|\|\mathbf{x}_{\mathrm{minor}}\| = \alpha$,
so we obtain the \textbf{admissible upper bound}:

\begin{equation}
  \boxed{
    \langle \mathbf{q}, \mathbf{x} \rangle
    \;\le\;
    \mathrm{dot}_{K}(\mathbf{q}, \mathbf{x}) + \alpha.
  }
  \label{eq:bound}
\end{equation}

The bound is tight when the minor components are identical; it is
independent of database size and holds for any $\alpha \in [0,1)$.

\subsection{Exact and Deterministic Two-Phase Search}
\label{sec:twophase}

\begin{figure}[t]
\begin{minipage}{\linewidth}
\begin{small}
\textbf{Algorithm 1} Two-phase exact search
\hrule\vspace{4pt}
\textbf{Input:} query $\mathbf{q}$, database $\{\mathbf{x}_i\}$,
  energy-split parameter $\alpha$ (serves as admissible bound)\\
\textbf{Output:} index of the exact nearest neighbour
\begin{enumerate}
  \item \textbf{Phase 1 (major scan, store scores):}
        Compute and store $s_i = \mathrm{dot}_{K}(\mathbf{q}, \mathbf{x}_i)$
        for all $i$. Let $b = \arg\max_i s_i$, breaking ties by index.
  \item Set $s_{\mathrm{best}} =
        \langle\mathbf{q}, \mathbf{x}_b\rangle$ (one full dot using stored $s_b$);
        set $\mathrm{best\_idx} \leftarrow b$.
  \item \textbf{Phase 2 (admissible-bound pruning):}
        For each $i \ne b$ in index order:\\
        \quad If $s_i + \alpha \le s_{\mathrm{best}}$: \textbf{skip}
              (entry $i$ provably cannot exceed $s_{\mathrm{best}}$).\\
        \quad Else: compute $c = s_i + \langle\mathbf{q}_{\mathrm{minor}}, \mathbf{x}_{i,\mathrm{minor}}\rangle$.\\
        \quad\quad If $c > s_{\mathrm{best}}$: $s_{\mathrm{best}} \leftarrow c$,
              $\mathrm{best\_idx} \leftarrow i$.
  \item Return $\mathrm{best\_idx}$.
\end{enumerate}
\hrule
\end{small}
\end{minipage}
\label{alg:hier}
\end{figure}

\begin{theorem}[Exactness]
\label{thm:exact}
For any database and query using descriptors satisfying Eq.~\eqref{eq:hn},
Algorithm~1 returns the same index as exhaustive full-vector search.
\end{theorem}
\begin{proof}
For any entry $i$ skipped in Phase~2:
the skip condition gives $s_i + \alpha \le s_{\mathrm{best}}$.
By Eq.~\eqref{eq:bound},
$\langle\mathbf{q},\mathbf{x}_i\rangle \le s_i + \alpha \le s_{\mathrm{best}}$,
so entry $i$ provably cannot exceed the current best.
Entry $b$ is fully evaluated first, and $s_{\mathrm{best}}$ is updated for every
non-skipped entry; hence it equals the maximum inner product over all
fully evaluated entries.
Since every skipped entry is known to lie at or below $s_{\mathrm{best}}$
without computing its full inner product, $s_{\mathrm{best}}$ attains the global
maximum; the returned index is therefore the global nearest neighbour.
\end{proof}

\begin{theorem}[Determinism]
\label{thm:det}
For a fixed database ordering and tie-breaking rule, Algorithm~1 produces a
unique output for each query, reproducible across runs on the same platform.
\end{theorem}
\begin{proof}
$\mathrm{dot}_K(\mathbf{q},\mathbf{x}_i)$ is a deterministic
matrix--vector product; ties are resolved by index, giving a total order.
The Phase-2 traversal, bound test, and maximisation use no randomness.
The output is a deterministic function of (database, query, database ordering).
\end{proof}

\noindent\textbf{Remark (ordering dependence in practice).}
The dependence on database ordering in Theorem~\ref{thm:det} arises
\emph{only} when two entries share an identical full inner product with the
query---an exact numerical tie resolved by index.
With 32-bit floating-point \hn{}-normalised descriptors optimised by gradient
descent, such ties occur with negligible probability, so the output is
effectively independent of database ordering for virtually all
(database, query) pairs.
In contrast, with \emph{quantised} descriptors (e.g.\ 8-bit integer or binary
codes), the discrete value space makes exact ties considerably more likely;
in those settings, database ordering becomes a meaningful part of the
determinism guarantee and should be explicitly fixed and documented.

\noindent\textbf{Note on Phase-2 scan order.}
Algorithm~1 traverses entries in \emph{index order} rather than sorted by
descending $s_i$.
A sorted scan would enable early \emph{loop} termination (once
$s_i + \alpha < s_{\mathrm{best}}$, all remaining entries are prunable), but
sorting adds $O(N\log N)$ overhead.
In the SoA layout the score buffer (2\,MB for $K{=}16$, $N{=}500{,}000$)
fits entirely in the $12$\,MB L3 cache, so the $N$ sequential score reads in
Phase-2 are cheap; the index-order design avoids sorting while incurring
negligible memory bandwidth cost.

Together, Theorems~\ref{thm:exact} and~\ref{thm:det} separate \hn{} from
ANN methods (HNSW, IVF-PQ) and from MRL pipelines built on such indexes:
those trade exactness for speed and produce run-varying results, whereas
\hn{} reproduces the exhaustive-search result every time.

% ── 4. Experimental Setup ─────────────────────────────────────────────────────
\section{Experimental Setup}
\label{sec:setup}

\paragraph{Backbone.}
\hardnet{}~\cite{mishchuk2017hardnet} with publicly available pre-trained
weights, modified by replacing the final L2 normalisation with
$\mathrm{HN}_\alpha$.

\paragraph{Training.}
All models are initialised from the pretrained \hardnet{} weights and
fine-tuned on the notredame split of the \ubc{} patch
dataset~\cite{brown2011discriminative} with TripletMarginLoss (margin~$1.0$),
Adam ($\mathrm{lr}{=}10^{-4}$), batch~size~32 (empirically chosen for the
fine-tuning regime; larger batches yielded no improvement).
We report two settings: $K{=}16$ (30~epochs) and $K{=}8$ (10~epochs);
epoch counts were determined by monitoring the training loss on a held-out
validation subset and stopping when the loss plateaued.
Because \hn{} fine-tuning only needs to shift energy into the major
subspace---not relearn full discriminability---and because \hardnet{}'s
original training data is no longer readily accessible (retraining on
substitute data degrades quality), a short fine-tuning schedule suffices
for both settings.

\paragraph{Dataset.}
The \ubc{} patch dataset~\cite{brown2011discriminative} provides three
independent patch collections: notredame (train), trevi and halfdome (test).
Pre-extracted patches are used directly, without a keypoint detector front-end.
Cross-scene evaluation measures generalisation under domain shift.

\paragraph{Metric.}
FPR@95: false positive rate at $95\%$ true positive rate.
Lower is better.

% ── 5. Results ────────────────────────────────────────────────────────────────
\section{Experiments}
\label{sec:experiments}

\subsection{Descriptor Quality}

\begin{table}[t]
\centering
\caption{FPR@95 on trevi and halfdome.
  \hn{} models: pretrained \hardnet{} weights fine-tuned on notredame with \hn{} normalisation.
  Lower is better.
  $\dagger$: pretrained baseline, no \hn{}.
  Best \hn{} result within each $K$ setting in \textbf{bold}.}
\label{tab:main}
\setlength{\tabcolsep}{4pt}
\begin{tabular}{llrr}
\toprule
Method & $\alpha$ & trevi & halfdome \\
\midrule
SIFT & --- & 0.0617 & 0.1415 \\
\hardnet{} (pretrained)$^\dagger$ & --- & 0.0062 & 0.0201 \\
\midrule
\hn{} $K{=}16$, $\alpha{=}0$    & 0     & 0.0109 & 0.0787 \\
\hn{} $K{=}16$, $\alpha{=}1/32$ & 1/32  & 0.0089 & 0.0653 \\
\hn{} $K{=}16$, $\alpha{=}1/16$ & 1/16  & 0.0078 & 0.0539 \\
\hn{} $K{=}16$, $\alpha{=}1/8$  & 1/8   & 0.0064 & 0.0462 \\
\hn{} $K{=}16$, $\alpha{=}1/4$  & 1/4   & \textbf{0.0056} & \textbf{0.0414} \\
\midrule
\hn{} $K{=}8$, $\alpha{=}0$    & 0     & 0.0447 & 0.2243 \\
\hn{} $K{=}8$, $\alpha{=}1/32$ & 1/32  & 0.0397 & 0.1837 \\
\hn{} $K{=}8$, $\alpha{=}1/16$ & 1/16  & 0.0339 & 0.1565 \\
\hn{} $K{=}8$, $\alpha{=}1/8$  & 1/8   & 0.0270 & 0.1186 \\
\hn{} $K{=}8$, $\alpha{=}1/4$  & 1/4   & \textbf{0.0164} & \textbf{0.0904} \\
\bottomrule
\end{tabular}
\end{table}

\paragraph{$K{=}16$.}
Quality improves monotonically with $\alpha$.
At $\alpha{=}0$ (pure major-only), FPR@95 is $0.0109$ on trevi.
At $\alpha{=}1/8$ the gap to the pretrained baseline is already negligible
($0.0002$), confirming that \hn{} normalisation preserves almost all
discriminative information at a substantial speed-up.
At $\alpha{=}1/4$, FPR@95 reaches $0.0056$---below the pretrained baseline
of $0.0062$---because \hn{} fine-tuning reshapes the descriptor geometry to
concentrate energy in the major subspace; this restructuring incidentally
improves cross-scene discriminability, consistent with known regularisation
effects of fine-tuning under constrained normalisation.
All configurations satisfy the exactness guarantee of Theorem~\ref{thm:exact}.
Figure~\ref{fig:roc16} shows the full ROC curves.

\begin{figure*}[htbp]
  \centering
  \includegraphics[width=\textwidth]{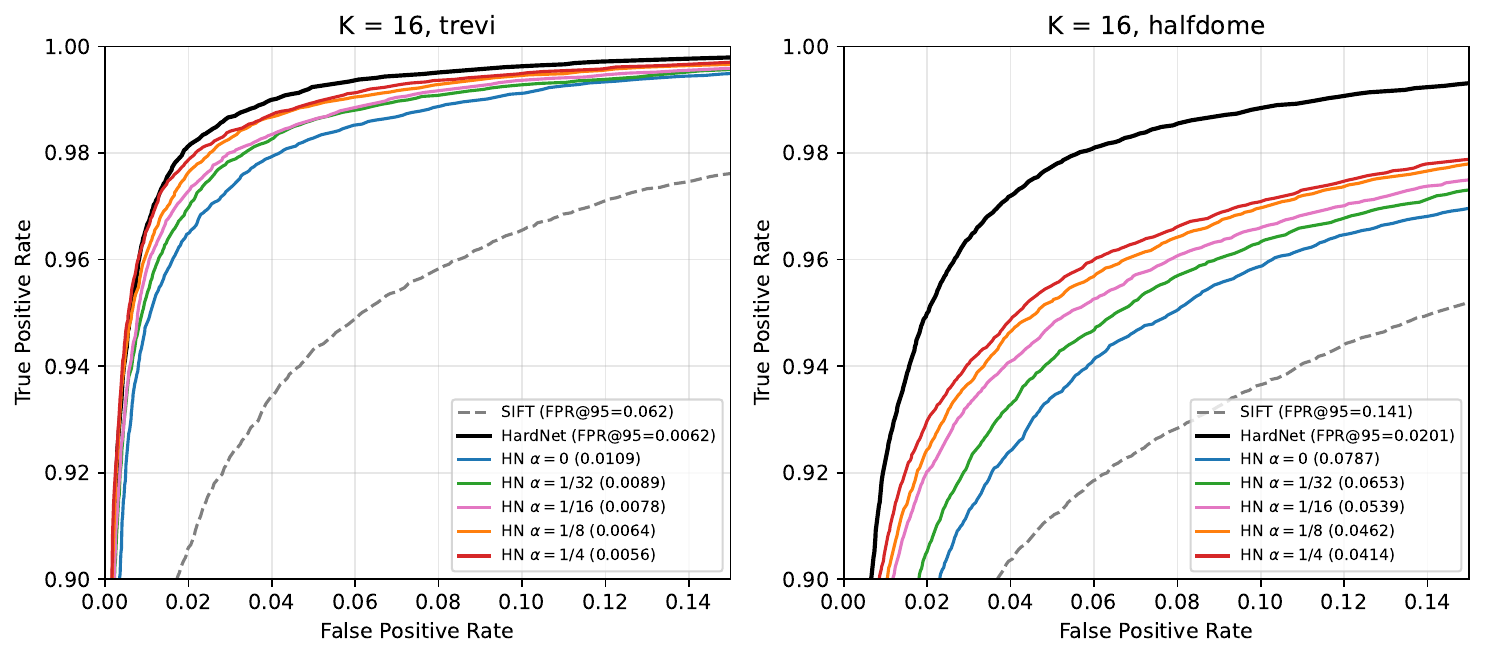}
  \caption{$K{=}16$ ROC curves on trevi (left) and halfdome (right) for all $\alpha$
           configurations vs.\ \hardnet{} baseline and SIFT. FPR@95 values in legend.}
  \label{fig:roc16}
\end{figure*}

\paragraph{$K{=}8$.}
Quality improves monotonically with $\alpha$, mirroring the $K{=}16$ trend.
At $\alpha{=}1/4$, FPR@95 reaches $0.0164$ on trevi and $0.0904$ on halfdome;
$\alpha{=}1/8$ gives $0.0270$ on trevi, already well below SIFT ($0.0617$).
The larger gap vs.\ the pretrained baseline reflects the greater challenge of
representing information with only an 8-dim major prefix.
Figure~\ref{fig:roc8} shows the full ROC curves.

\begin{figure*}[htbp]
  \centering
  \includegraphics[width=\textwidth]{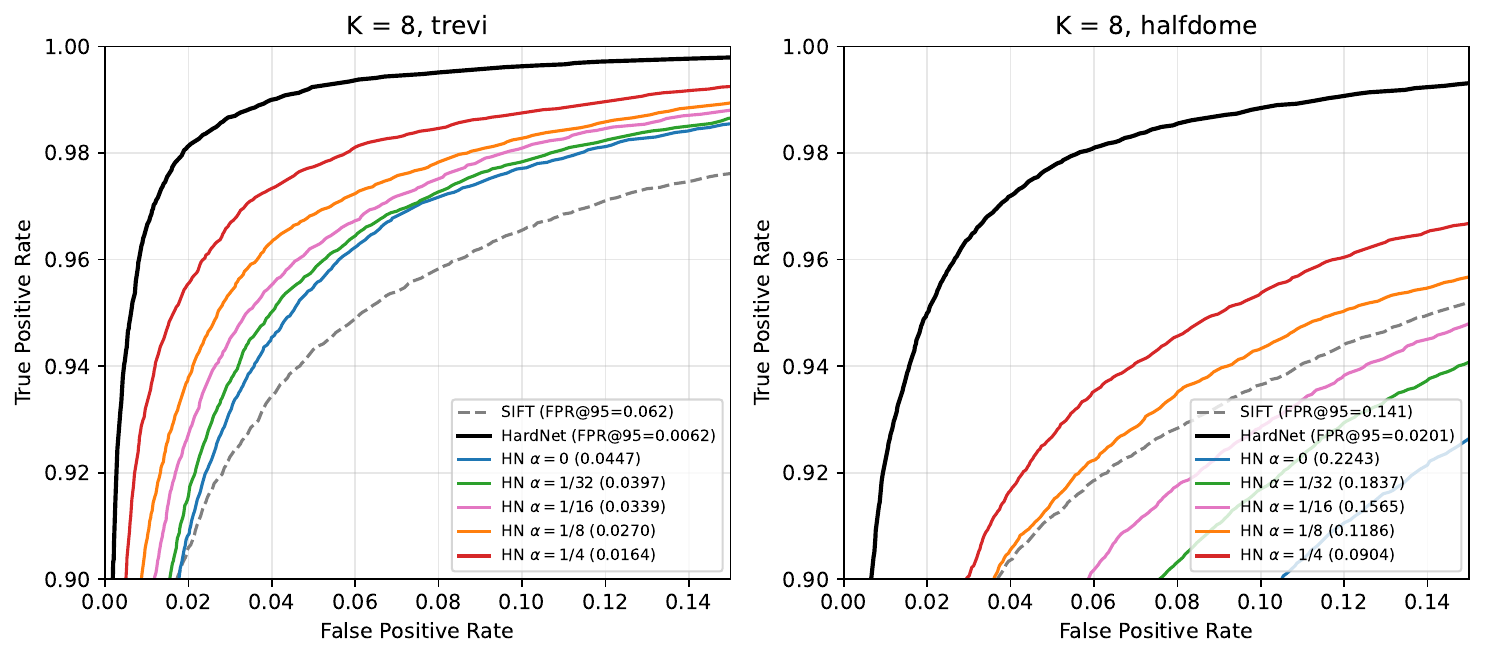}
  \caption{$K{=}8$ ROC curves on trevi (left) and halfdome (right) for all $\alpha$
           configurations vs.\ \hardnet{} baseline and SIFT.}
  \label{fig:roc8}
\end{figure*}

\paragraph{Effect of $\alpha$.}
Across both $K$ settings, larger $\alpha$ consistently reduces FPR: with
greater $\alpha$, more energy is allocated to the $(128{-}K)$-dim minor
component, spreading discriminative information across additional dimensions
and enabling a richer overall descriptor representation.
At $K{=}16$, the FPR improvement from $\alpha{=}1/32$ to $\alpha{=}1/4$ spans
only $0.0033$ on trevi; at $K{=}8$ the same range spans $0.0233$,
confirming gradual and stable improvement in both settings.
Figure~\ref{fig:fpralpha} shows this trend across both scenes.

\paragraph{Note on halfdome vs.\ trevi.}
FPR@95 is consistently higher on halfdome across all configurations.
This is not specific to \hn{}: the pretrained \hardnet{} baseline already
shows $0.0201$ on halfdome vs.\ $0.0062$ on trevi, reflecting greater
photometric and viewpoint variation in the halfdome scene.
The relative advantage of \hn{} configurations over SIFT is maintained on
both scenes.

\begin{figure*}[t]
  \centering
  \includegraphics[width=\textwidth]{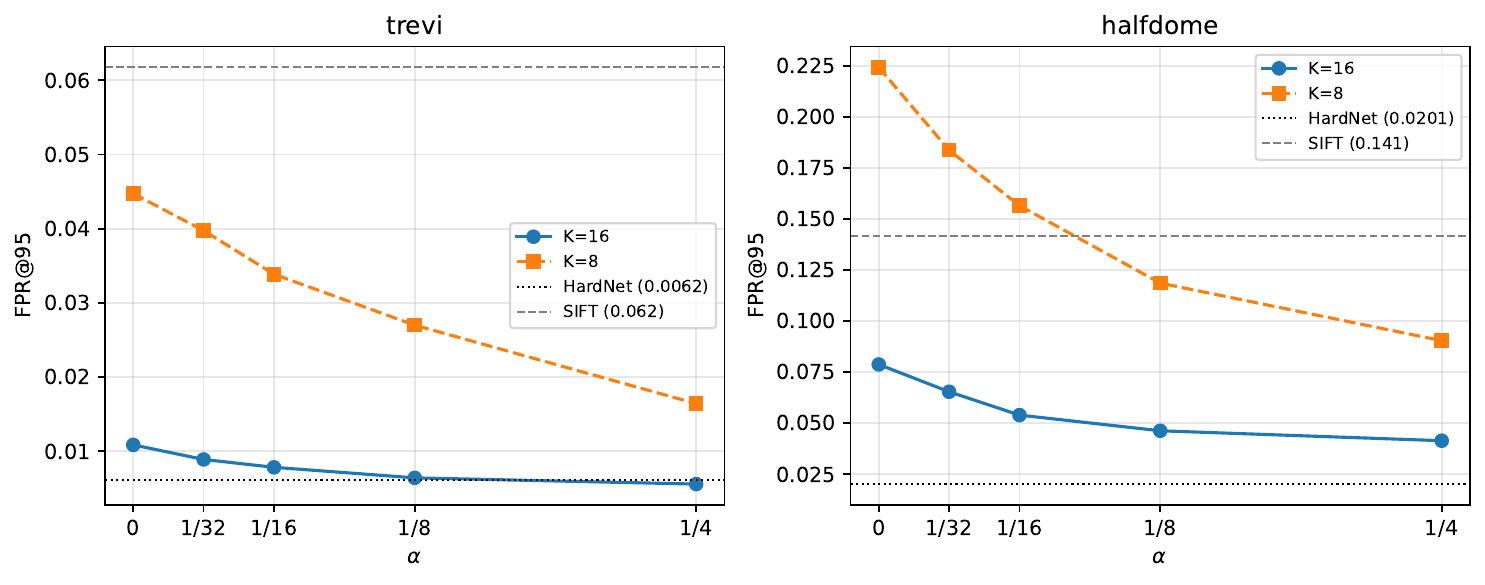}
  \caption{FPR@95 vs.\ energy-split parameter $\alpha$ for $K{=}16$ and $K{=}8$
           on trevi (left) and halfdome (right).
           Dashed lines: \hardnet{} pretrained and SIFT baselines.}
  \label{fig:fpralpha}
\end{figure*}

\subsection{Quality–Speed Trade-off}

\begin{table}[h]
\centering
\caption{Quality vs.\ speed at $\alpha{=}1/8$, $N{=}500{,}000$ entries
  (FPR@95 from Table~\ref{tab:main}; speed from Table~\ref{tab:speed}).
  \hn{} models fine-tuned on notredame with \hn{} normalisation.
  SpeedHier: vs.\ brute-force SoA 128-dim.}
\label{tab:kcomp}
\setlength{\tabcolsep}{2pt}
\begin{tabular}{lrrrr}
\toprule
Method & \multicolumn{2}{c}{trevi} & \multicolumn{2}{c}{halfdome} \\
\cmidrule(lr){2-3}\cmidrule(lr){4-5}
 & FPR@95 & Speed & FPR@95 & Speed \\
\midrule
\hardnet{} pretrained     & 0.0062 & $1.0\times$ & 0.0201 & $1.0\times$ \\
\hn{} $K{=}16$, $\alpha{=}1/8$ & 0.0064 & $7.2\times$ & 0.0462 & $5.8\times$ \\
\hn{} $K{=}8$,  $\alpha{=}1/8$ & 0.0270 & $10.1\times$ & 0.1186 & $5.4\times$ \\
\bottomrule
\end{tabular}
\end{table}

$K{=}16$ trades a small quality cost ($+0.0002$ FPR on trevi) for a $7.2\times$
speed-up with full exactness.
At $\alpha{=}1/8$, $K{=}8$ achieves higher throughput on trevi ($10.1\times$
vs.\ $7.2\times$) at a larger quality cost; on halfdome the elevated
Phase-2\% ($6.6\%$ vs.\ $2.8\%$) reverses this advantage ($5.4\times$
vs.\ $5.8\times$), making $K{=}8$ most useful when Phase-2\% stays low.
Figure~\ref{fig:pareto} visualises the full Pareto frontier.

\begin{figure*}[t]
  \centering
  \includegraphics[width=\textwidth]{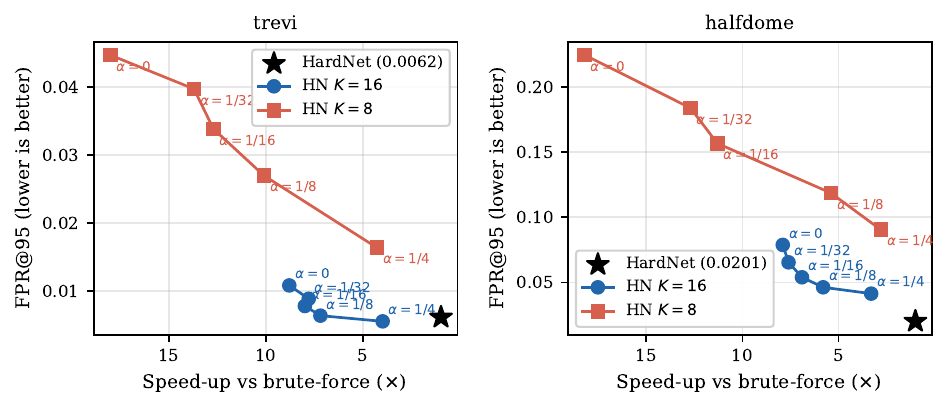}
  \caption{Quality--speed Pareto frontier on trevi (left) and halfdome (right).
           Each point is one $(K,\alpha)$ configuration; the star marks the
           \hardnet{} pretrained baseline at $1\times$ speed.
           Both axes: lower FPR@95 and higher speed-up are better.}
  \label{fig:pareto}
\end{figure*}

% ── 6. Speed Analysis ─────────────────────────────────────────────────────────
\section{Speed Analysis}
\label{sec:speed}

\subsection{Memory Layout: SoA}

In the conventional \emph{Array of Structures} (\textsc{AoS}) layout,
the $K$-dim major scan must load all 128 floats per entry ($512$~bytes),
wasting cache.
A \emph{Structure of Arrays} (\textsc{SoA}) layout stores majors and minors
in separate contiguous arrays.
For $K{=}8$ and $N{=}500{,}000$, the major array is $16$~MB---versus
$256$~MB for the full descriptor bank---reducing memory bandwidth by $16\times$.
For $K{=}16$, the major array is $32$~MB.

\subsection{Wall-clock Benchmark}

\begin{table}[t]
\centering
\caption{Speed benchmark (\hn{} fine-tuned on notredame),
  SoA layout, $N{=}500{,}000$, 200 test queries, 10 runs averaged (Intel Core Ultra~7~256V, 12\,MB L3;
  Windows~11, MSVC~19.50, \texttt{/O2}).
  SpeedHier: vs.\ brute-force SoA 128-dim.
  Phase2\%: fraction requiring full 128-dim evaluation.
  $\alpha{=}1/4$ rows measured in a separate session; all other rows from a single idle-PC session.}
\label{tab:speed}
\setlength{\tabcolsep}{3.5pt}
\begin{tabular}{llrrrr}
\toprule
$K$ & $\alpha$ & \multicolumn{2}{c}{trevi} & \multicolumn{2}{c}{halfdome} \\
\cmidrule(lr){3-4}\cmidrule(lr){5-6}
 & & SpeedHier & Phase2\% & SpeedHier & Phase2\% \\
\midrule
16 & 0     & $8.8\times$  & 0.0\% & $7.9\times$  & 0.0\% \\
16 & 1/32  & $7.8\times$  & 0.4\% & $7.6\times$  & 0.5\% \\
16 & 1/16  & $8.0\times$  & 0.6\% & $6.9\times$  & 1.1\% \\
16 & 1/8   & $7.2\times$  & 1.2\% & $5.8\times$  & 2.8\% \\
16 & 1/4   & $4.0\times$  & 5.5\% & $3.3\times$  & 6.5\% \\
\midrule
 8 & 0     & $18.0\times$ & 0.0\% & $18.2\times$ & 0.0\% \\
 8 & 1/32  & $13.7\times$ & 0.4\% & $12.7\times$ & 0.6\% \\
 8 & 1/16  & $12.7\times$ & 0.6\% & $11.3\times$ & 1.0\% \\
 8 & 1/8   & $10.1\times$ & 1.5\% & $5.4\times$  & 6.6\% \\
 8 & 1/4   & $4.3\times$  & 6.6\% & $2.8\times$  & 11.5\% \\
\bottomrule
\end{tabular}
\end{table}

At $K{=}8$, the 16\,MB major array is $16\times$ smaller than the full
descriptor bank (256\,MB), reducing memory pressure and achieving up to
$18.2\times$ speed-up ($18.0\times$ on trevi, $18.2\times$ on halfdome);
the 16\,MB buffer is close to the $12$\,MB L3 boundary, so partial caching
amplifies the bandwidth advantage beyond the nominal $16\times$ ratio.
At $K{=}16$, peak speed is $8.8\times$ (trevi, $\alpha{=}0$); the larger
32\,MB major array exceeds L3 and is fully bandwidth-bound, consistent with
the $8\times$ bandwidth ratio of 32\,MB vs.\ 256\,MB.
For $\alpha{\le}1/8$, Phase-2 rates stay below $2\%$ on trevi and below
$7\%$ on halfdome; at $\alpha{=}1/4$, Phase-2 rises to at most $6.6\%$ on
trevi and $11.5\%$ on halfdome ($K{=}8$), so the exactness overhead
(Theorem~\ref{thm:exact}) is negligible in practice.
A notable exception to the decreasing-speed-with-$\alpha$ trend is $K{=}8$,
$\alpha{=}0$: at $\alpha{=}0$, \hn{} normalisation sets all 120 minor components
to exactly zero, so Phase~2 never accesses the 240\,MB minor array at all.
Any $\alpha{>}0$ forces even a handful of Phase-2 entries (${\approx}2{,}000$
at $\alpha{=}1/32$) to issue random reads into the 240\,MB minor array,
causing L3 misses and a sharp drop from $18.0\times$ to $13.7\times$.
Minor fluctuations in Table~\ref{tab:speed} (e.g.\ $K{=}16$ trevi,
$\alpha{=}1/32$ vs.\ $1/16$) are within timing noise; the overall trend
is decreasing speedup with increasing $\alpha$.
The largest scene-dependent variation is at $K{=}8$, $\alpha{=}1/8$:
trevi achieves $10.1\times$ while halfdome achieves $5.4\times$.
Halfdome's greater photometric variation makes the 8-dimensional major
subspace less discriminative, raising Phase-2\% from $1.5\%$ (trevi) to
$6.6\%$ (halfdome); the resulting ${\sim}33{,}000$ random accesses into
the $240$\,MB minor array saturate the $12$\,MB L3 cache.
With Phase-1 at $4$\,M FLOPs ($N{\times}8$) and Phase-2 at $4$\,M FLOPs
($33{,}000{\times}120$), the FLOP model predicts $64/8 = 8\times$ speed-up.
The observed $5.4\times$ falls below this prediction because the $33{,}000$
random Phase-2 accesses are latency-bound (L3 cache misses), not
throughput-bound---while the lower-Phase-2 trevi case remains
bandwidth-dominated at $10.1\times$.

\paragraph{FLOP analysis.}
At $K{=}16$, $\alpha{=}1/32$ on trevi (Phase2~$= 0.4\%$):
{\small\begin{align*}
\text{Full 128-dim:}         &\quad N{\times}128 = 64\text{M} \\
\text{Phase~1 (major + store):}  &\quad N{\times}16  =  8\text{M} \\
\text{Phase~2 (check stored):}   &\quad {\approx}\,0\text{M (L3)} \\
\text{Phase~2 (full, 0.4\%):}   &\quad 0.004N{\times}(112{+}\tau) \approx 0.2\text{M}{+}2{,}000\tau \\
\text{Total:}                    &\quad \approx 8.2\text{M}{+}2{,}000\tau
\end{align*}}
The FLOP ratio predicts $\approx7.8\times$ speed-up, consistent with the
observed $7.8\times$ in Table~\ref{tab:speed},
where $\tau$ denotes the DRAM read latency per Phase-2 entry in FLOP-equivalent
units; at $2{,}000$ entries, $2{,}000\tau$ is small relative to $8.2$\,M.
Phase~1 writes the $N$ major scores (2\,MB) to a score buffer that resides in
L3 cache; Phase~2 reads this buffer sequentially at negligible cost.
The rare Phase-2 full evaluations ($0.4\%$, ${\approx}2{,}000$ entries) access
the 224\,MB minor array, which is entirely cold at this point: the SoA layout
ensures Phase~1 never touches the minor array, so every Phase-2 access is a
first-time DRAM read. At only $2{,}000$ such accesses, the total latency
overhead is negligible.

At $\alpha{=}0$, \hn{} normalisation sets the minor components to exactly
zero, so Phase~2 requires no dot-product computation at all---only the
2\,MB score buffer is read for comparison.
The bandwidth ratio of the 32\,MB major array to the 256\,MB full descriptor
bank is $8\times$; the observed $8.8\times$ ($K{=}16$, trevi) is consistent
with this prediction.

% ── 7. Discussion ─────────────────────────────────────────────────────────────
\section{Discussion}
\label{sec:discussion}

\paragraph{Why determinism matters.}
Modern retrieval systems benefit directly from reproducible results.
Caching requires that the same query always returns the same item;
A/B testing requires that both branches search the same result space;
regression tracking requires that a code change---not hardware or thread
scheduling---explains any difference in top-1 output.
ANN indexes (HNSW, IVF-PQ) and MRL pipelines built on them fail all three:
graph-build order and entry-point selection cause query-time non-determinism.
\hn{} satisfies all three by construction (Theorem~\ref{thm:det}).

\paragraph{Why exactness matters.}
ANN search introduces silent errors: the returned item may \emph{look}
correct but not be the true nearest neighbour.
In safety-critical matching (e.g.\ forensic image identification, visual
place recognition for autonomous navigation), this is unacceptable.
\hn{} provides a mathematical certificate of correctness (Theorem~\ref{thm:exact}):
the returned result is always the true nearest neighbour.

\paragraph{Speed comparison with HNSW.}
Reviewers naturally ask how \hn{} speed compares to HNSW, which achieves
far higher throughput via approximation.
We benchmarked both methods in the same C++ binary on the same machine
using $N{=}99{,}000$ UBC Trevi descriptors (fine-tuned \hn{}, 128-dim,
inner-product space; HNSW: $M{=}16$, ef\textsubscript{construction}$=200$;
HNSW was benchmarked on trevi only).
On halfdome at $N{=}99{,}000$, \hn{} $K{=}16$ achieves $12.1$--$12.9\times$
across $\alpha{=}0$--$1/8$ and $10.3\times$ at $\alpha{=}1/4$,
consistent with the trevi pattern.
Results for trevi are summarised in Table~\ref{tab:hnsw}.

\begin{table}[h]
\centering
\footnotesize
\setlength{\tabcolsep}{2pt}
\caption{C++ benchmark: \hn{} (exact) vs.\ HNSW (approx.) on UBC Trevi descriptors
  at three database sizes (fine-tuned \hn{}, $K{=}16$, all $\alpha$, 128-dim,
  Intel Core Ultra~7~256V, 12\,MB L3, MSVC \texttt{/O2}).
  Speed-up vs.\ brute-force SoA 128-dim; \textbf{bold}: HN exact exceeds all HNSW settings at $N{=}1{,}000$.
  ``---'': not measured at that size.}
\label{tab:hnsw}
\resizebox{\columnwidth}{!}{%
\begin{tabular}{lcccc}
\toprule
Method & Recall@1\textsuperscript{†} & \multicolumn{3}{c}{Speed-up} \\
\cmidrule(lr){3-5}
 & & $N{=}1\text{K}$ & $N{=}5\text{K}$ & $N{=}99\text{K}$ \\
\midrule
Brute-force & 100\% (exact) & $1.0\times$ & $1.0\times$ & $1.0\times$ \\
\hn{} $K{=}16$, $\alpha{=}0$ (exact)    & 100\% (proven) & ---          & ---          & $12.7\times$ \\
\hn{} $K{=}16$, $\alpha{=}1/32$ (exact) & 100\% (proven) & ---          & ---          & $12.6\times$ \\
\hn{} $K{=}16$, $\alpha{=}1/16$ (exact) & 100\% (proven) & ---          & ---          & $12.6\times$ \\
\hn{} $K{=}16$, $\alpha{=}1/8$ (exact)  & 100\% (proven) & $\mathbf{9.5\times}$ & $9.5\times$ & $12.5\times$ \\
\hn{} $K{=}16$, $\alpha{=}1/4$ (exact)  & 100\% (proven) & ---          & ---          & $10.4\times$ \\
\midrule
HNSW ef$=10$  & 95--98\%\textsuperscript{*} & $7.2\times$ & $27.6\times$ & $349\times$ \\
HNSW ef$=50$  & 99.9\%\textsuperscript{*} & $2.3\times$ & $7.8\times$ & $111.9\times$ \\
HNSW ef$=100$ & 100\%\textsuperscript{*}  & $1.4\times$ & $4.8\times$ & $63.5\times$ \\
HNSW ef$=200$ & 100\%\textsuperscript{*}  & $0.9\times$ & $2.9\times$ & $35.3\times$ \\
HNSW ef$=500$ & 100\%\textsuperscript{*}  & $0.5\times$ & $1.3\times$ & $15.8\times$ \\
\bottomrule
\end{tabular}}
\caption*{\textsuperscript{†}Recall shown for $N{=}99{,}000$; smaller $N$ recall is similar or higher.\\
  \textsuperscript{*}Empirical; no formal guarantee.
  On synthetic data at $N{=}500\text{K}$, ef$=50$ drops to $\approx\!94\%$ recall.}
\end{table}

Four observations follow.
\emph{(i) HNSW is considerably faster at large $N$.}
At ef$=50$, HNSW achieves $111.9\times$ vs.\ \hn{}'s $12.5\times$ ($K{=}16$) at $N{=}99{,}000$.
At the larger $N{=}500{,}000$ scale (Table~\ref{tab:speed}), $K{=}16$ drops to $7.2\times$:
the 32\,MB major buffer exceeds the 12\,MB L3 cache, incurring memory-bandwidth overhead.
$K{=}8$, by contrast, maintains $18.0\times$ at $N{=}500{,}000$ because its 16\,MB
major buffer is close to the L3 boundary and benefits from partial caching.
Even at ef$=500$, HNSW reaches $15.8\times$.
\emph{(ii) HNSW's recall is scale-dependent and unguaranteed.}
At $N{=}99{,}000$, ef$=100$ and above achieve 100\% empirical recall on this dataset.
On synthetic data at $N{=}500{,}000$, ef$=50$ drops to $\approx\!94\%$ recall,
and the ef required for full recall grows with $N$.
The safe ef value cannot be determined without a ground-truth oracle.
\emph{(iii)} \hn{} \emph{provides a formal correctness certificate.}
Theorem~\ref{thm:exact} proves that \hn{} never returns a wrong result,
independent of dataset, size, or hardware ordering.
For applications requiring auditable exactness---authentication, forensics,
safety-critical navigation---this guarantee is irreplaceable, and no ef
setting can provide it through HNSW.
\emph{(iv)} \hn{}\emph{'s speed-up is N-independent; HNSW's grows with $N$.}
At $N{=}5{,}000$, \hn{} achieves $9.5\times$---only modestly lower than $12.5\times$
at $N{=}99{,}000$---while HNSW ef$=50$ drops to $7.8\times$, falling below \hn{}.
At $N{=}1{,}000$ (typical AR map: 300--1000 3-D landmarks), \hn{} still achieves
$9.5\times$ while every HNSW setting falls below it: ef$=10$ reaches only $7.2\times$
with 98\% recall, and ef$\ge 200$ is \emph{slower than brute-force}.
\hn{}'s phase-2 pruning rate is a property of the descriptor distribution
alone, while HNSW's advantage grows as $N/\log N$ and disappears at small $N$.

\paragraph{Choice of $\alpha$ and $K$.}
The quality--speed trade-off is monotone: larger $K$ or larger $\alpha$
improves quality and reduces speed.
The quality--speed curve for $K{=}16$ has a pronounced elbow near
$\alpha{=}1/8$: for $\alpha\in\{1/32,1/16,1/8\}$ the HN speed-up stays
nearly constant ($12.5$--$12.6\times$ on trevi; Table~\ref{tab:hnsw}) while FPR@95
falls from $0.0089$ to $0.0064$.
Increasing to $\alpha{=}1/4$ reduces the speed-up to $10.4\times$ on trevi
(Table~\ref{tab:hnsw}; $10.3\times$ on halfdome) at $N{=}99{,}000$
and $4.0\times$/$3.3\times$ at $N{=}500{,}000$ (Table~\ref{tab:speed}),
for a smaller quality gain ($0.0064\to0.0056$).
We therefore recommend $K{=}16$, $\alpha{=}1/8$ as the default operating
point, at the knee of the $N{=}99{,}000$ quality--speed curve
(Table~\ref{tab:hnsw}), corresponding to $7.2\times$ speed-up at
$N{=}500{,}000$ (Table~\ref{tab:kcomp}).
$K{=}8$, $\alpha{=}0$ maximises throughput ($18.0\times$ on trevi / $18.2\times$ on halfdome, FPR@95~$= 0.0447$ / $0.2243$).
Practitioners can select any $(K, \alpha)$ point on the Pareto frontier
without sacrificing exactness or determinism.

\paragraph{Relation to PCA.}
Standard PCA applied post-hoc to \hardnet{} descriptors yields no useful
low-dimensional compression.
\hardnet{} is trained with hard-negative triplet loss to maximise
discriminability, which tends to distribute information \emph{uniformly} across all 128
dimensions; PCA therefore finds no dominant principal components, and
truncating to $K$ dimensions discards roughly $1 - K/128$ of the discriminative
signal irrespective of which $K$ are chosen.
\hn{} circumvents this by shaping the representation \emph{during training}.
Gradient flow is weighted $(1{-}\alpha)$ toward the major component and $\alpha$
toward the minor, so the network actively encodes its most discriminative
features into the first $K$ dimensions.
This is analogous to \emph{active PCA}: rather than passively discovering
principal components in a fixed embedding, \hn{} trains the embedding so that a
known prefix \emph{becomes} the principal subspace---one whose information
concentration is a learned property, not a post-hoc accident.

\paragraph{Power-of-two $\alpha$.}
The experimental $\alpha$ values $\{1/32, 1/16, 1/8, 1/4\}$ are all powers of
two, a natural choice: in quantised implementations, the admissible threshold
$+\alpha = 2^{-n}$ maps to a single integer unit shift, and the relative scale
between major and minor components aligns with integer bit widths.
We leave a systematic study of fixed-point \hn{} to future work.

\paragraph{Soft ANN mode via threshold relaxation.}
Algorithm~1 can be extended to a tunable ANN mode by replacing the admissible
threshold $\alpha$ with a smaller value $\beta < \alpha$ in the skip condition.
At $\beta{=}\alpha$ the algorithm is provably exact (Theorem~\ref{thm:exact});
decreasing $\beta$ toward zero trades exactness for additional pruning.
This single-parameter spectrum---exact at one end, aggressive ANN at the
other---requires no structural change to the index.
We leave a systematic evaluation of this trade-off to future work.

\paragraph{Combination with ANN indexes.}
In principle, the $K$-dim major vectors could be indexed (e.g.\ $k$-d tree, PQ)
to make Phase-1 sub-linear.
The lower dimensionality ($K{=}8$ or $16$) does reduce the curse of
dimensionality compared to full $D$-dim search.
However, \hn{} training maximises discriminability within the major subspace,
which tends to distribute major vectors \emph{uniformly} across the $K$-dim
hypersphere---precisely the distribution for which $k$-d trees degrade toward
linear scan and PQ quantisation loses the most quality.
In practice, the current speed-ups derive from the SoA major-first layout
(Phase~1 sequentially scans the compact major buffer) and from the admissible
bound eliminating redundant Phase-2 full evaluations; whether sub-linear
indexing on the major would further help remains an open question.

% ── 8. Conclusion ─────────────────────────────────────────────────────────────
\section{Conclusion}
\label{sec:conclusion}

We presented a patch descriptor retrieval method that is
\textbf{exact}---provably identical to exhaustive search---and
\textbf{deterministic}---reproducible across every run and hardware
configuration for a fixed database ordering (see Theorem~\ref{thm:det}).
The enabling technique, Hierarchical Normalization (\hn{}), concentrates
$1{-}\alpha$ of a \hardnet{} descriptor's energy into a compact $K$-dim
major prefix, yielding an admissible upper bound that licenses
branch-and-bound pruning.
On the \ubc{} benchmark (notredame $\to$ trevi, halfdome), $K{=}16$,
$\alpha{=}1/8$ achieves FPR@95~$= 0.0064$ on trevi at $7.2\times$ speed-up;
$K{=}8$, $\alpha{=}1/32$ achieves $13.7\times$ (trevi) / $12.7\times$ (halfdome).
In both cases the admissible bound provides a mathematical correctness
certificate absent from ANN-based and MRL-based retrieval systems.
Future directions include $K{=}8$ training with larger batch sizes and
hard-negative mining (to reduce the quality gap with $K{=}16$), higher-$K$
ablations, the soft ANN mode ($\beta{<}\alpha$) trade-off, and sub-linear
Phase-1 search via lightweight ANN indexing on the major subspace.

% ── References ────────────────────────────────────────────────────────────────
\bibliographystyle{unsrtnat}

\end{document}